\documentclass[draft]{new_tlp}
\usepackage{mathptmx}

\usepackage[obeyDraft]{todonotes}\presetkeys{todonotes}{inline}{}
\usepackage{pgfplots}\usetikzlibrary{plotmarks}\pgfplotsset{compat=1.14}

\usepackage[ruled,linesnumbered,vlined]{algorithm2e}

\usepackage{amssymb}
\usepackage{booktabs}
\usepackage{comment}
\usepackage{enumitem}
\usepackage{multirow}
\usepackage{url}
\usepackage{thm-restate}

  \title[Theory and Practice of Logic Programming]
        {Model enumeration in propositional circumscription via unsatisfiable core analysis}

  \author[M. Alviano]
         {MARIO ALVIANO\\
         Department of Mathematics and Computer Science, University of Calabria, Italy\\
         \email{alviano@mat.unical.it}}

\jdate{March 2003}
\pubyear{2003}
\pagerange{\pageref{firstpage}--\pageref{lastpage}}
\doi{S1471068401001193}

\newtheorem{example}{Example}

\begin{document}
%\nocite{*}% includes all entries of BibTeX database into the list of references.

\label{firstpage}

\maketitle

  \begin{abstract}
    Many practical problems are characterized by a preference relation over admissible solutions, where preferred solutions are minimal in some sense.
    For example, a preferred diagnosis usually comprises a minimal set of reasons that is sufficient to cause the observed anomaly.
    Alternatively, a minimal correction subset comprises a minimal set of reasons whose deletion is sufficient to eliminate the observed anomaly.
    Circumscription formalizes such preference relations by associating propositional theories with minimal models.
    The resulting enumeration problem is addressed here by means of a new algorithm taking advantage of unsatisfiable core analysis.
    Empirical evidence of the efficiency of the algorithm is given by comparing the performance of the resulting solver, \textsc{circumscriptino}, with \textsc{hclasp}, \textsc{camus\_mcs}, \textsc{lbx} and \textsc{mcsls} on the enumeration of minimal models for problems originating from practical applications.
  \end{abstract}


%\tableofcontents

\section{Introduction}

Circumscription \cite{DBLP:journals/ai/McCarthy80} is a nonmonotonic logic formalizing common sense reasoning by means of a second order semantics.
Intuitively, circumscription allows to express that some things are as expected unless otherwise specified, a property that cannot be expressed in monotonic languages such as first order logic.
More specifically, the idea of circumscription is to minimize the extension of some predicates.
In the special case of propositional theories, which are the focus of the present paper, the simplest form of circumscription essentially selects subset minimal models.
In the form introduced by \citeANP{DBLP:journals/ai/Lifschitz86}~\citeyear{DBLP:journals/ai/Lifschitz86}, instead, some atoms are used to group interpretations, and other atoms are subject to minimization.

Many practical problems are characterized by a preference relation over admissible solutions.
For example, when analyzing a faulty system, several diagnoses are usually possible, and the debugging process can be improved by focusing on those diagnoses comprising a minimal set of reasons that is sufficient to cause the observed anomaly \cite{DBLP:conf/aadebug/PereiraDA93,DBLP:journals/jair/JannachSS16}.
In this case the preference relation is given by the subset containment relation.
The same preference relation can be used to recover the faulty system:
a correction subset is a set of reasons whose deletion is sufficient to eliminate the observed anomaly, and intuitively the debugging process has to focus on \emph{minimal correction subsets} \cite{DBLP:conf/aaai/Junker04,DBLP:conf/ijcai/Marques-SilvaHJPB13}.
It turns out that such practical problems have a natural representation in the framework of circumscription.

The notion of a minimal correction subset received particular attention from the scientific community in recent years, in particular in the context of propositional logic \cite{DBLP:conf/sat/Marques-SilvaP14,DBLP:conf/ijcai/MenciaPM15,DBLP:conf/sat/MenciaIPM16}.
In this case, a minimal correction subset of an unsatisfiable propositional theory $T$ is a subset $S$ of $T$ such that $T \setminus S$ is satisfiable, while $T \setminus S'$ is still unsatisfiable for all $S' \subset S$.
A natural representation of this problem in circumscription is obtained by replacing each formula $\phi \in T$ with $\phi \vee x_\phi$, where $x_\phi$ is a fresh variable subject to minimization.
Models of the obtained circumscribed theory represent minimal correction subsets of $T$:
in fact, the minimal correction subset associated with a model $I$ is $\{\phi \in T \mid x_\phi \in I\}$.

It is therefore not a surprise that the enumeration of models satisfying some minimality condition has been already addressed in the literature \cite{DBLP:journals/tplp/KaminskiSSV13,DBLP:journals/jar/LiffitonS08,DBLP:conf/ecai/FaberVCG16}.
In particular, \citeANP{DBLP:journals/tplp/KaminskiSSV13} focused on the computation of minimal intervention strategies in logical signaling networks representing biological scenarios;
among the several techniques they presented, the most efficient takes advantage of the domain heuristics supported by the solver \textsc{hclasp}.
\citeANP{DBLP:journals/jar/LiffitonS08} instead focused on the enumeration of minimal correction subsets in order to subsequently compute minimal unsatisfiable subsets (or minimal unsatisfiable cores);
the first computational task is accomplished by the solver \textsc{camus\_mcs} by enumerating models of size $n$ that do not contain previously reported models, for increasing values of $n$.
Finally, \citeANP{DBLP:conf/ecai/FaberVCG16} presented a general purpose algorithm to achieve minimal model enumeration by iteratively enumerating cardinality minimal models of the input theory that do not contain previously reported models;
enumeration of cardinality minimal models is usually achieved by computing a first model of minimal cardinality, and then by enumerating all models of that size.

Cardinality minimal models of propositional theories are solutions of the computational problem known as MaxSAT.
For problems originating from practical applications, the most efficient MaxSAT algorithms are based on unsatisfiable core analysis \cite{DBLP:journals/constraints/MorgadoHLPM13};
among them are \textsc{oll} \cite{DBLP:conf/iclp/AndresKMS12} and \textsc{one} \cite{DBLP:conf/ijcai/AlvianoDR15}.
In particular, \textsc{one} can be seen as a simplification of \textsc{oll}, and essentially modifies the processed theory by enforcing the satisfaction of all but one formulas in the analyzed unsatisfiable core.
A natural question is therefore whether these algorithms can be directly adapted to achieve minimal model enumeration.
In particular, such an algorithm would not rely on any global condition on the size of the computed models, which is an advantage because global constraints of this kind may exponentially deteriorate the performance of a solver \cite{DBLP:conf/sat/BacchusN14}.

This paper provides a positive answer to the above question:
an algorithm for the enumeration of models of a circumscribed theory is presented;
the algorithm takes advantage of the unsatisfiable core analysis provided by \textsc{one}, and enumerates models sorted by size.
This is a property shared with the approaches proposed by \citeANP{DBLP:journals/jar/LiffitonS08}, and \citeANP{DBLP:conf/ecai/FaberVCG16}. 
However, differently from them, and like the strategy adopted by \citeANP{DBLP:journals/tplp/KaminskiSSV13}, the new algorithm smoothly runs on an incremental solver, meaning that new formulas are introduced during its execution, but none of them need to be subsequently removed.

A prototype solver implementing the proposed algorithm is also presented.
It is called \textsc{circumscriptino}, and relies on the incremental SAT solver \textsc{glucose} \cite{DBLP:conf/ijcai/AudemardS09} for computing models and unsatisfiable cores of the processed propositional theory.
The prototype is evaluated empirically on four testcases representing the logical signaling networks analyzed by \citeANP{DBLP:journals/tplp/KaminskiSSV13}:
in order to enumerate minimal intervention strategies with \textsc{circumscriptino}, the instances processed by \textsc{hclasp} are translated into classical propositional theories by means of the convenient tool \textsc{lp2sat} \cite{DBLP:conf/birthday/JanhunenN11}.
Three of the tested instances are solved by \textsc{hclasp} and \textsc{circumscriptino} in a few seconds.
Many other testcases are obtained from the SAT Solver Competitions \cite{DBLP:journals/aim/JarvisaloBRS12}.
In particular, the prototype is evaluated on the enumeration of minimal correction subsets for unsatisfiable instances of the MUS Special track.
On these instances, the performance of \textsc{circumscriptino} is superior than those of the specialized solvers \textsc{camus\_mcs} \cite{DBLP:journals/jar/LiffitonS08}, \textsc{lbx} \cite{DBLP:conf/ijcai/MenciaPM15} and \textsc{mcsls} \cite{DBLP:conf/ijcai/Marques-SilvaHJPB13}, as well as of \textsc{hclasp} \cite{DBLP:conf/aaai/GebserKROSW13}.

\section{Background}

Let $\mathcal{A}$ be a fixed, countable set of \emph{atoms} including $\bot$.
A \emph{literal} is an atom possibly preceded by the connective $\neg$.
For a literal $\ell$, let $\overline{\ell}$ denote its \emph{complementary literal}, that is, $\overline{p} = \neg p$ and $\overline{\neg p} = p$ for all $p \in \mathcal{A}$;
for a set $L$ of literals, let $\overline{L}$ be $\{\overline{\ell} \mid \ell \in L\}$.
Moreover, for a set $L$ of literals and a set $A$ of atoms, the \emph{restriction} of $L$ to symbols in $A$ is $L|_A := L \cap (A \cup \overline{A})$.

\emph{Formulas} are defined as usual by combining atoms and the connectives $\neg$, $\wedge$, $\vee$, $\rightarrow$.
In addition, a formula can be a \emph{cardinality constraint} of the following form:
\begin{equation}\label{eq:cc}
    \ell_1 + \cdots + \ell_n \geq k
\end{equation}
where $\ell_1,\ldots,\ell_n$ are literals, $n \geq 1$ and $k \geq 0$.
A \emph{theory} is a set $T$ of formulas including $\neg \bot$;
the set of atoms occurring in $T$ is denoted by $\mathit{atoms}(T)$.

\begin{example}\label{ex:run:1}
The following theories will be used as running examples:
\[
    \begin{array}{llll}
        T_1 := \{\neg \bot,\quad a \vee x_0,\quad \neg a \vee b \vee x_1,\quad \neg a \vee \neg b \vee x_2\}; \\
        T_2 := T_1 \cup \{r \rightarrow x_0 \wedge x_1 \wedge x_2\}; \\
        T_3 := T_1 \cup \{\neg x_0 + \neg x_1 + \neg x_2 + y_1 + y_2 \geq 2,\quad y_2 \rightarrow y_1\}.
    \end{array}
\]
In particular, $T_3$ will be used in Section~\ref{sec:enum} as an example of unsatisfiable core analysis.
\hfill$\blacksquare$
\end{example}

An \emph{assignment} is a set $A$ of literals such that $A \cap \overline{A} = \emptyset$.
An \emph{interpretation} for a theory $T$ is an assignment $I$ such that $(I \cup \overline{I}) \cap \mathcal{A} = \mathit{atoms}(T)$.
Relation $\models$ is defined as usual:
for $p \in \mathcal{A}$, $I \models p$ if $p \in I$;
for $\phi$ and $\psi$ formulas, $I \models \neg \phi$ if $I \not\models \phi$,
$I \models \phi \wedge \psi$ if $I \models \phi$ and $I \models \psi$,
$I \models \phi \vee \psi$ if $I \models \phi$ or $I \models \psi$, and
$I \models \phi \rightarrow \psi$ if $I \models \psi$ whenever $I \models \phi$;
for $\phi$ of the form (\ref{eq:cc}), $I \models \phi$ if $|I \cap \{\ell_1,\ldots,\ell_n\}| \geq k$;
for a theory $T$, $I \models T$ if $I \models \phi$ for all $\phi \in T$.
$I$ is a \emph{model} of a theory $T$ if $I \models T$.
Let $\mathit{models}(T)$ denote the set of models of $T$.
(Models will be also represented by the set of their atoms, as their negative literals are implicit.)

\begin{example}[Continuing Example~\ref{ex:run:1}]\label{ex:run:2}
$T_1$ has 16 models, including the following:
$\{x_0\}$,
$\{x_0, b\}$, 
$\{x_1, a\}$, 
$\{x_2, a, b\}$, 
$\{x_0, x_1\}$,
$\{x_0, x_1,$ $a\}$, and
$\{x_0, x_1, b\}$.
These are also models of $T_2$ (where $r$ is false);
$T_2$ additionally admits, for any $X \subseteq \{a,b\}$,
$\{x_0, x_1, x_2, r\} \cup X$.
Regarding $T_3$, variables $y_1,y_2$ can be used to constrain the number of false atoms among $\{x_0,x_1,x_2\}$.
For example, models extending the assignment $\{\neg y_1, \neg y_2\}$ must assign false to at least two atoms in $\{x_0,x_1,x_2\}$;
these models are precisely
$\{x_0\}$,
$\{x_0, b\}$, 
$\{x_1, a\}$, and
$\{x_2, a, b\}$. 
Finally, note that $T_1$, $T_2$ and $T_3$ have no models extending the assignment $\{\neg x_0, \neg x_1, \neg x_2\}$. 
\hfill$\blacksquare$
\end{example}

\emph{Circumscription} applies to a theory $T$ and sets $P,Z$ of atoms;
atoms in $P$ are subject to minimization, while atoms in $Z$ are irrelevant.
Formally, relation $\leq^{PZ}$ is defined as follows:
for $I,J$ interpretations of $T$, $I \leq^{PZ} J$ if both $I|_{\mathcal{A} \setminus (P \cup Z)} = J|_{\mathcal{A} \setminus (P \cup Z)}$ and $I \cap P \subseteq J \cap P$.
$I \in \mathit{models}(T)$ is a \emph{preferred model} of $T$ with respect to $\leq^{PZ}$ if there is no $J \in \mathit{models}(T)$ such that $I \not\leq^{PZ} J$ and $J \leq^{PZ} I$.
Let $\mathit{CIRC}(T,P,Z)$ denote the set of preferred models of $T$ with respect to $\leq^{PZ}$.

\begin{example}[Continuing Example~\ref{ex:run:2}]\label{ex:run:3}
Let $P$ be $\{x_0,x_1,x_2\}$, and $Z \subseteq \{a,b\}$.
$\mathit{CIRC}(T_1,P,Z)$ and $\mathit{CIRC}(T_2,P,Z \cup \{r\})$ are
$\{\{x_0\}$,
$\{x_0, b\}$, 
$\{x_1, a\}$, 
$\{x_2, a, b\}\}$,
while $\mathit{CIRC}(T_2,P,Z)$ additionally includes
$\{x_0,x_1,x_2,r\} \cup X$, for all $X \subseteq \{a,b\}$.
Indeed, if $r$ is not irrelevant, then $I \not\leq^{PZ} \{x_0,x_1,x_2,r\} \cup X$ for all $I \in \{\{x_0\}$,
$\{x_0, b\}$, 
$\{x_1, a\}$, 
$\{x_2, a, b\}\}$, and all $X \subseteq \{a,b\}$.
Regarding $T_3$, note that $\mathit{CIRC}(T_1,\{y_1,y_2\},Z)$ is again
$\{\{x_0\}$,
$\{x_0, b\}$, 
$\{x_1, a\}$, 
$\{x_2, a, b\}\}$,
for any $Z \subseteq \{a,b\}$.
\hfill$\blacksquare$
\end{example}

\section{Model enumeration}\label{sec:enum}

The computational problem addressed in this paper is the enumeration of models of a circumscribed theory, that is, the enumeration of $\mathit{CIRC}(T,P,Z)$.
The proposed algorithm takes advantage of modern solvers for checking satisfiability of propositional theories.
This is achieved by means of function $\mathit{solve}$, whose input is a theory $T$ and a set $A$ of literals called \emph{assumptions}.
The function searches for a model $I$ of $T$ such that $A \subseteq I$.
If such an $I$ exists, tuple $(\mathit{true}, I, B, -)$ is returned, where $B \subseteq I$ is the set of branching literals used to compute $I$.
Otherwise, a tuple $(\mathit{false}, -, -, C)$ is returned, where $C \subseteq A$ is such that $T \cup \{\ell \mid \ell \in C\}$ has no models;
in this case, $C$ is called \emph{unsatisfiable core}.

\begin{example}[Continuing Example~\ref{ex:run:3}]\label{ex:run:4}
The result of $\mathit{solve}(T_1,\{\neg x_0, \neg x_1, \neg x_2\})$ is $(\mathit{false}, -, -, \{\neg x_0, \neg x_1, \neg x_2\})$, that is, set $\{\neg x_0, \neg x_1,$ $\neg x_2\}$ is an unsatisfiable core.
On the other hand, $\mathit{solve}(T_3,\{\neg y_1, \neg y_2\})$ returns $(\mathit{true}, I, B, -)$, where $I \in \{\{x_0\}$,
$\{x_0, b\}$, 
$\{x_1, a\}$, 
$\{x_2, a, b\}\}$
(and $B \subseteq I$).
\hfill$\blacksquare$
\end{example}

\begin{algorithm}[t]
    \caption{Model enumeration for $\mathit{CIRC}(T,P,Z)$}\label{alg:circ}
    $O := P$;\qquad $V := \mathit{atoms}(T) \setminus \{\bot\}$;\qquad $R := V \setminus (P \cup Z)$\;
    \Repeat{{\bf not} $\mathit{sat}$ {\bf and} $C = \emptyset$}{
        $(\mathit{sat},I,-,C) := \mathit{solve}(T,\overline{O})$\tcp*{try to falsify all objective literals} \label{alg:circ:solve}
        \uIf(\tcp*[f]{minimal model found}){$\mathit{sat}$}{
            $\mathit{enumerate}(T, V, I|_{P \cup R} \cup (O \setminus P))$\tcp*{fix $P,R$, and disable new constraints}
            $T := T \cup \{\bigwedge{I|_R} \rightarrow \bigvee{\overline{P \cap I}}\}$\tcp*{add blocking clause}
        }
        \ElseIf(\tcp*[f]{unsatisfiable core analysis}){$C \neq \emptyset$}{
            Let $C$ be $\{\neg x_0,\ldots,\neg x_n\}$ ($n \geq 0$), and $y_1,\ldots,y_n$ be fresh variables\;
            $O := (O \setminus \{x_0,\ldots,x_n\}) \cup \{y_1,\ldots,y_n\}$\;
            $T := T \cup \{\neg x_0 + \cdots + \neg x_n + y_1 + \cdots + y_n \geq n\} \cup \{y_{i} \rightarrow y_{i-1} \mid i \in [2..n]\}$\;
        }
    }(\tcp*[f]{empty unsatisfiable core implies no more models})
\end{algorithm}

Algorithm~\ref{alg:circ} implements model enumeration for $\mathit{CIRC}(T,P,Z)$.
%to be more precise, instead of minimizing the true literals in $P$, the algorithm equivalently maximizes true literals in $\overline{P}$.
The following sets are used by the algorithm:
$O$ for the objective literals, that is, those to minimize, initially $P$;
$V$ for the visible atoms, that is, those in the input theory $T$;
$R$ for the (other) relevant atoms, that is, those not in $P \cup Z$.
The algorithm iteratively searches for a model of $T$ falsifying all objective literals.
If a nonempty unsatisfiable core $C$ is returned, it is processed according to the \textsc{one} algorithm (lines~8--10):
objective literals in the unsatisfiable core are replaced by $|C|-1$ new objective literals (not already occurring in $T$), and the theory $T$ is extended with new formulas enforcing the truth of at least $|C|-i$ literals in $C$ whenever the $i$-th new objective literal is false.
On the other hand, if a model is found, say $I$, it is guaranteed to be minimal with respect to $P$.
In this case, the interpretation of atoms in $P$ and $R$ is fixed, and all formulas introduced in line~10 are satisfied by assuming the truth of all objective literals introduced in line~9.
All models extending these assumptions are then enumerated, for example by means of the polyspace algorithm introduced by \citeANP{DBLP:conf/lpnmr/GebserKNS07}~\citeyear{DBLP:conf/lpnmr/GebserKNS07}, here given in terms of assumptions \cite{DBLP:conf/aiia/AlvianoD16}.
After that, a blocking clause is added to $T$, so to discard all interpretations $J$ such that $I \leq^{PZ} J$ (including $I$ itself):
intuitively, any model $J$ such that $J|_R = I|_R$ is forced to falsify at least one atom of $P$ that is interpreted as true by $I$.
The algorithm terminates as soon as an empty unsatisfiable core is returned by function $\mathit{solve}$, meaning that all models in $\mathit{CIRC}(T,P,Z)$ have been computed.

\begin{procedure}[t]
    \caption{enumerate($T$, $V$, $A$)}\label{alg:enum}
    $\mathit{push}(A,\neg \bot)$;\qquad $F := \emptyset$\tcp*{initialize assumptions and flipped literals}
    \While(\tcp*[f]{there are still assumptions to be flipped}){$\mathit{top}(A) \neq \bot$}{
        $(\mathit{sat},I,B,C) := \mathit{solve}(T,A)$\tcp*{search $I \in \mathit{models}(T)$ such that $A \subseteq I$}
        \uIf(\tcp*[f]{found $I$ using branching literals $B$}){$\mathit{sat}$}{
            \textbf{print} $I \cap V$\tcp*{report model}
            \lFor(\tcp*[f]{extend $A$ with new branching literals}){$\ell \in B \setminus A$}{$\mathit{push}(A,\ell)$}
        }
        \Else(\tcp*[f]{found unsatisfiable core $C \subseteq A$}){
            \lWhile(\tcp*[f]{backjump}){$\mathit{top}(A) \neq \neg \bot$ \textbf{and} $\mathit{top}(A) \notin C$}{$F := F \setminus \{\mathit{pop}(A)\}$}
        }
        \lWhile(\tcp*[f]{remove flipped assumptions}){$\mathit{top}(A) \in F$}{$F := F \setminus \{\mathit{pop}(A)\}$}
        $\mathit{push}(A, \overline{\mathit{pop}(A)})$;\qquad
        $F := F \cup \{\mathit{top}(A)\}$\tcp*{flip top assumption}
    }
\end{procedure}

\begin{example}[Continuing Example~\ref{ex:run:4}]\label{ex:run:5}
Let $T$ be $T_1$, $P$ be $\{x_0,x_1,x_2\}$, and $Z$ be empty.
Algorithm~\ref{alg:circ} starts by setting $O$ and $R$ respectively to $\{x_0,x_1,x_2\}$ and $\{a,b\}$.
The first call to function $\mathit{solve}$ returns $(\mathit{false}, -, -, \{\neg x_0, \neg x_1, \neg x_2\})$, and therefore the unsatisfiable core $\{\neg x_0, \neg x_1, \neg x_2\}$ is analyzed (lines~7--10):
set $O$ becomes $\{y_1, y_2\}$, where $y_1$ and $y_2$ are fresh variables, and $T$ is extended with 
$\neg x_0 + \neg x_1 + \neg x_2 + y_1 + y_2 \geq 2$, and
$y_2 \rightarrow y_1$.
Note that $T$ is now $T_3$.
The second call to function $\mathit{solve}$ then returns $(\mathit{true}, \{y_1, y_2\} \cup I, -, -)$, where $I \in \{\{x_0\}$,
$\{x_0, b\}$, 
$\{x_1, a\}$, 
$\{x_2, a, b\}\}$.
Say that $I$ is $\{x_0\}$;
the enumeration procedure is called with assumptions $\{x_0, \neg x_1, \neg x_2, \neg a, \neg b, y_1, y_2\}$ (recall that negative literals are implicit in $I$, hence $I|_{P \cup R}$ is $\{x_0, \neg x_1, \neg x_2, \neg a, \neg b\}$).
In this case, model $I$ is computed again (in linear time with modern solvers), and the enumeration procedure terminates.
Theory $T$ is extended with the blocking clause $\neg a \wedge \neg b \rightarrow \neg x_0$, and a new model is computed, say $\{x_0, b\}$.
Again, since all atoms are relevant, the enumeration procedure terminates reporting only $\{x_0, b\}$ itself.
Theory $T$ is extended with the blocking clause $\neg a \wedge b \rightarrow \neg x_0$, and a new model is computed, say $\{x_1, a\}$.
Theory $T$ is extended with the blocking clause $a \wedge \neg b \rightarrow \neg x_1$, and model $\{x_2, a, b\}$ is computed.
Finally, the blocking clause $a \wedge b \rightarrow \neg x_2$ is added, and the empty unsatisfiable core is returned by function $\mathit{solve}$.
Hence, all models of $\mathit{CIRC}(T_1,\{x_0,x_1,x_2\},\emptyset)$ were reported, and the algorithm terminates.

For $Z$ being $\{a,b\}$, $R$ is empty and the algorithm behaves differently starting from the call to the enumeration procedure.
Indeed, for $I = \{x_0\}$, the assumptions are $\{x_0, \neg x_1, \neg x_2, y_1, y_2\}$, and the procedure reports two models, namely $\{x_0\}$ and $\{x_0, b\}$.
Moreover, the blocking clause added to $T$ is $\neg x_0$, so that the next model returned by function $\mathit{solve}$ must be either $\{x_1, a\}$ or $\{x_2, a, b\}$.
The associated blocking clauses are $\neg x_1$ and $\neg x_2$, and after adding them the empty unsatisfiable core is returned by function $\mathit{solve}$, so that the algorithm can terminate.

For $T$ being $T_2$, and $Z$ being $\{a,b\}$, $R$ is $\{r\}$ and the algorithm behaves differently starting from the second call to function $\mathit{solve}$.
Indeed, in this case the returned model may also be $\{x_0,x_1,x_2,r\} \cup X$, for $X \subseteq \{a,b\}$.
Say that $I$ is $\{x_0,x_1,x_2,r\}$;
the enumeration procedure is called with assumptions $\{x_0,x_1,x_2,r, y_1,y_2\}$, and models $\{x_0,x_1,x_2,r\} \cup X$, for $X \subseteq \{a,b\}$, are reported.
After that, the blocking clause $r \rightarrow \neg x_0 \vee \neg x_1 \vee \neg x_2$ is added to $T$, so that a model $I \in \{\{x_0\}$,
$\{x_0, b\}$, 
$\{x_1, a\}$, 
$\{x_2, a, b\}\}$
can be returned by function $\mathit{solve}$.
From this point, the algorithm continues as for $\mathit{CIRC}(T_1,\{x_0,x_1,x_2\},\{a,b\})$, with the only difference that the added blocking clauses are
$\neg r \rightarrow \neg x_0$,
$\neg r \rightarrow \neg x_1$, and
$\neg r \rightarrow \neg x_2$.
\hfill$\blacksquare$
\end{example}

\subsection{Correctness}

The following main theorem is proved in this section.

\begin{restatable}{theorem}{ThmMain}\label{thm:main}
Let $T$ be a theory, and $P,Z$ be sets of atoms.
Algorithm~\ref{alg:circ} enumerates all models in $\mathit{CIRC}(T,P,Z)$, and the number of iterations of the repeat-until loop is bounded by $|\mathit{models}(T)| + |P|$.
\end{restatable}

The above theorem is proved by showing that the algorithm is correct at each iteration:
when lines~4--6 are executed, models of the processed theory are either reported or satisfy the added blocking clause;
when lines~7--10 are executed, all models are preserved.

First of all, recall that a model is possibly represented as the set of its atoms (i.e., negative literals are ignored).
Hence, for sets $S,S'$ of models, we will write $S = S'$ if $\{I \cap \mathcal{A} \mid I \in S\} = \{I \cap \mathcal{A} \mid I \in S'\}$, even if $S$ and $S'$ are models of theories with different atoms.
The following lemma states that procedure $\mathit{enumerate}(T,V,A)$ computes all models of $T$ extending the assignment $A$.

\begin{restatable}{lemma}{LemEnum}\label{lem:enum}
Let $T$ be a theory, $V$ be a set of atoms, and $A$ be a set of literals.
Procedure $\mathit{enumerate}(T,V,A)$ computes $\{I \cap V \mid I \in \mathit{models}(T \cup \{\ell \mid \ell \in A\})\}$.
\end{restatable}
\begin{proof}
The procedure given in this paper extends the one presented by \citeANP{DBLP:conf/aiia/AlvianoD16}~\citeyear{DBLP:conf/aiia/AlvianoD16} with the possibility of providing in input a set $A$ of assumptions.
In order to extend the correctness of the enumeration procedure presented by \citeANP{DBLP:conf/aiia/AlvianoD16}, we have only to note that assumptions in $A$ are protected by literal $\neg \bot$ (pushed on line~1 of the procedure), so that they are never flipped or removed by the procedure.
All models of $T$ extending the provided assumptions are therefore reported, only printing true atoms among those in $V$.
\hfill
\end{proof}

Among the assumptions passed to procedure $\mathit{enumerate}$ are the variables introduced by \textsc{one}.
They are assumed true so to restore the original theory.

\begin{restatable}{lemma}{LemDisable}\label{lem:disable}
Let $T$ be a theory, and $T'$ be $T \cup \{\neg x_0 + \cdots + \neg x_n + y_1 + \cdots + y_n \geq n\} \cup \{y_{i} \rightarrow y_{i-1} \mid i \in [2..n]\}$, for $n \geq 0$.
If $y_i \notin \mathit{atoms}(T)$ for $i \in [1..n]$, $\mathit{models}(T) = \{I|_{\mathit{atoms}(T)} \mid I \in \mathit{models}(T' \cup \{y_i \mid i \in [1..n]\})\}$.
\end{restatable}
\begin{proof}
If $I$ is a model of $T$, then $I \cup \{y_i \mid i \in [1..n]\}$ is a model of $T' \cup \{y_i \mid i \in [1..n]\}$.
Moreover, any model $J$ of $T' \cup \{y_i \mid i \in [1..n]\}$ satisfies $J \models T$ (because $T \subseteq T'$).
\hfill
\end{proof}

Hence, when procedure $\mathit{enumerate}$ is invoked on line~5 of Algorithm~\ref{alg:circ}, since the initial assumptions comprise $O \setminus P$, and because of Lemma~\ref{lem:disable}, all models of $T$ extending the assignment $I|_{P \cup R}$ are computed.
These models are then discarded by the blocking clause added in line~6, as formalized by the following claim.

\begin{restatable}{lemma}{LemBC}\label{lem:bc}
Let $I$ be a model in $\mathit{CIRC}(T,P,Z)$, $R$ be $\mathit{atoms}(T) \setminus (P \cup Z \cup \{\bot\})$, and $\phi$ be $\bigwedge I|_R \rightarrow \bigvee \overline{P \cap I}$.
It holds that $\mathit{CIRC}(T,P,Z) = \mathit{CIRC}(T \cup \{\phi\},P,Z) \cup \mathit{models}(T \cup \{\ell \mid \ell \in I|_{P \cup R}\})$.
\end{restatable}
\begin{proof}
$(\subseteq)$
Consider $J \in \mathit{CIRC}(T,P,Z)$.
We distinguish two cases:
\begin{itemize}[leftmargin=15pt,labelsep=5pt]
\item[1.]
$J \models \phi$.
Hence, $J \in \mathit{models}(T \cup \{\phi\})$.
Let $J' \in \mathit{models}(T \cup \{\phi\})$ be such that $J' \leq^{PZ} J$.
Since $J \in \mathit{CIRC}(T,P,Z)$ by assumption, $J \leq^{PZ} J'$ holds, and therefore $J \in \mathit{CIRC}(T \cup \{\phi\},P,Z)$.

\item[2.]
$J \not\models \phi$.
Hence, $J \models \bigwedge{I|_R}$ and $J \not\models \bigvee{\overline{P \cap I}}$.
Note that $I|_R = I|_{\mathit{atoms}(T) \setminus (P \cup Z \cup \{\bot\})} = I|_{\mathcal{A} \setminus (P \cup Z)}$, and therefore $J \models \bigwedge{I|_R}$ implies $I|_{\mathcal{A} \setminus (P \cup Z)} = J|_{\mathcal{A} \setminus (P \cup Z)}$.
Moreover, $J \not\models \bigvee{\overline{P \cap I}}$ implies $I \cap P \subseteq J \cap P$.
From $I|_{\mathcal{A} \setminus (P \cup Z)} = J|_{\mathcal{A} \setminus (P \cup Z)}$ and $I \cap P \subseteq J \cap P$, we have $I \leq^{PZ} J$.
Since $J \in \mathit{CIRC}(T,P,Z)$ by assumption, $J \leq^{PZ} I$ holds.
Hence, $J|_{P \cup R} = I|_{P \cup R}$, which implies $J \in \mathit{models}(T \cup \{\ell \mid \ell \in I|_{P \cup R}\})$.
\end{itemize}

\medskip
\noindent
$(\supseteq)$
We distinguish two cases:
\begin{itemize}[leftmargin=15pt,labelsep=5pt]
\item[1.]
$J \in \mathit{CIRC}(T \cup \{\phi\},P,Z)$.
We have $J \in \mathit{models}(T)$.
Let $J' \in \mathit{models}(T)$ be such that $J' \leq^{PZ} J$.
We shall show that $J' \in \mathit{models}(T \cup \{\phi\})$, which implies $J \leq^{PZ} J'$ and therefore $J \in \mathit{CIRC}(T,P,Z)$.
$J' \leq^{PZ} J$ implies $J'|_{\mathcal{A} \setminus (P \cup Z)} = J|_{\mathcal{A} \setminus (P \cup Z)}$ and $J' \cap P \subseteq J \cap P$.
From $J \models \phi$, either $J \not\models \bigwedge{I|_R}$, or $J \models \bigvee{\overline{P \cap I}}$.
Hence, $J' \not\models \bigwedge{I|_R}$, or $J' \models \bigvee{\overline{P \cap I}}$, that is, $J' \models \phi$ and then $J' \in \mathit{models}(T \cup \{\phi\})$.

\item[2.]
$J \in \mathit{models}(T \cup \{\ell \mid \ell \in I|_{P \cup R}\})$.
Since $J \models \bigwedge{I|_{P \cup R}}$, the symmetric difference of $I$ and $J$ is a subset of $Z$, which implies $J \leq^{PZ} I$ (as well as $I \leq^{PZ} J$).
Since $I \in \mathit{CIRC}(T,P,Z)$ by assumption, we can conclude that $J \in \mathit{CIRC}(T,P,Z)$.
\end{itemize}
This conclude the proof of the lemma.
\hfill
\end{proof}

The following lemma states that the application of \textsc{one} transforms the original problem into an equivalent problem.

\begin{restatable}{lemma}{LemOne}\label{lem:one}
Let $T$ be a theory, and $P,Z$ be sets of atoms.
If $\{x_0,\ldots,x_n\} \subseteq P$ ($n \geq 0$) is such that $\mathit{models}(T \cup \{\neg x_i \mid i \in [0..n]\}) = \emptyset$, then
$\mathit{CIRC}(T,P,Z) = \mathit{CIRC}(T',P',Z')$, where 
$T' = T \cup \{\neg x_0 + \cdots + \neg x_n + y_1 + \cdots + y_n \geq n\} \cup \{y_i \rightarrow y_{i-1} \mid i \in [2..n]\}$, 
$P' = (P \setminus \{x_0,\ldots,x_n\}) \cup \{y_1,\ldots,y_n\}$,
$Z' = Z \cup \{x_0,\ldots,x_n\}$,
and $y_1,\ldots,y_n$ are fresh variables.
\end{restatable}
\begin{proof}
Let $\mathit{ext}(I)$ be $I \cup \{y_i \mid i \in [1..n], |I \cap \{x_0,\ldots,x_n\}| > i\}$, and $\mathit{red}(I)$ be $I|_{\mathit{atoms}(T)}$.

\medskip
\noindent
$(\subseteq)$
Let $I \in \mathit{CIRC}(T,P,Z)$.
By construction, $\mathit{ext}(I) \models T'$.
Let $J$ be such that $J \models T'$ and $J \leq^{P'Z'} \mathit{ext}(I)$.
In order to have $\mathit{ext}(I) \in \mathit{CIRC}(T',P',Z')$, we shall show that $\mathit{ext}(I) \leq^{P'Z'} J$.
We have $\mathit{red}(J) \leq^{PZ} I$, and combining with $I \in \mathit{CIRC}(T,P,Z)$ we conclude $I \leq^{PZ} \mathit{red}(J)$.
The previous finally implies $\mathit{ext}(I) \leq^{P'Z'} J$, and we are done.

\medskip
\noindent
$(\supseteq)$
Let $I \in \mathit{CIRC}(T',P',Z')$.
By construction, $\mathit{red}(I) \models T$.
Let $J$ be such that $J \models T$ and $J \leq^{PZ} \mathit{red}(I)$.
In order to have $\mathit{red}(I) \in \mathit{CIRC}(T,P,Z)$, we shall show that $\mathit{red}(I) \leq^{PZ} J$.
We have $\mathit{ext}(J) \leq^{P'Z'} I$, and combining with $I \in \mathit{CIRC}(T',P',Z')$ we conclude $I \leq^{P'Z'} \mathit{ext}(J)$.
The previous finally implies $\mathit{red}(I) \leq^{PZ} J$, and we are done.
\hfill
\end{proof}

Finally, termination of the algorithm is guaranteed because Algorithm~\ref{alg:circ} executes lines~7--10 unless there is $I \in \mathit{CIRC}(T,P,Z)$ such that $|P \cap I| = |P|-|O|$;
otherwise, lines~4--6 are executed, and at least one model is discarded by the added blocking clause.
This argument also provides the desired bound on the iterations of the repeat-until loop.

\begin{proof}[Proof of Theorem~\ref{thm:main}]
Let $R$ be $\mathit{atoms}(T) \setminus (P \cup Z \cup \{\bot\})$, as in Algorithm~\ref{alg:circ}.
Let $T_i,O_i$ ($i \geq 0$) be the values of variables $T,O$ at iteration $i$ of Algorithm~\ref{alg:circ}.
Let $Z_i$ be $Z \cup (O_i \setminus P)$.
We use induction on $i$ to show the following proposition:
\begin{equation}\label{eq:thm:main:1}
    \textit{If } J \in \mathit{CIRC}(T,P,Z) \textit{ and } J \models T_i \textit{ then } J \in \mathit{CIRC}(T_i,O_i,Z_i).
\end{equation}
The base case is trivial as $T_0 = T$, $O_0 = P$ and $Z_0 = Z$.
Assume the proposition for $i \geq 0$ in order to show that it holds for $i+1$.
Let $J \in \mathit{CIRC}(T,P,Z)$ be such that $J \models T_{i+1}$.
Note that $J \models T_i$ as well, and therefore $J \in \mathit{CIRC}(T_i,O_i,Z_i)$ because of the induction hypothesis.
We now distinguish two cases depending on the outcome of function $\mathit{solve}(T_i,\overline{O_i})$:
\begin{itemize}[leftmargin=15pt,labelsep=5pt]
\item[1.]
$T_{i+1}$ is $T_i \cup \{\bigwedge{I|_R} \rightarrow \bigvee{\overline{P \cap I}}\}$, for some model $I$.
Hence, $I \in \mathit{CIRC}(T_i,O_i,Z_i)$ because $I \cap O_i = \emptyset$.
Moreover, $O_{i+1} = O_i$ and $Z_{i+1} = Z_i$, so that we can apply Lemma~\ref{lem:bc} to conclude $J \in \mathit{CIRC}(T_{i+1},O_{i+1},Z_{i+1}) \cup \mathit{models}(T_i \cup \{\ell \mid \ell \in I|_{P \cup R}\})$.
But $J \notin \mathit{models}(T_i \cup \{\ell \mid \ell \in I|_{P \cup R}\})$ because $J \models T_{i+1}$ by assumption.
Hence, $J \in \mathit{CIRC}(T_{i+1},O_{i+1},Z_{i+1})$.

\item[2.]
$T_{i+1}$ is $T_i \cup \{\neg x_0 + \cdots + \neg x_n + y_1 + \cdots + y_n \geq n\} \cup \{y_{j} \rightarrow y_{j-1} \mid j \in [2..n]\}$, for some unsatisfiable core $\{\neg x_0, \ldots, \neg x_n\}$.
Hence, $O_{i+1} := (O_i \setminus \{x_0,\ldots,x_n\}) \cup \{y_1,\ldots,y_n\}$, and we can apply Lemma~\ref{lem:one} to conclude $\mathit{CIRC}(T_i,O_i,Z_i) = \mathit{CIRC}(T_{i+1},O_{i+1},Z_{i+1})$.
Therefore, $J \in \mathit{CIRC}(T_{i+1},O_{i+1},Z_{i+1})$.
\end{itemize}
This completes the proof of (\ref{eq:thm:main:1}).

We can also note that the number of iterations of Algorithm~\ref{alg:circ} is bounded by $|\mathit{models}(T)| + |P|$ because for all $i \geq 0$ either $|\mathit{models}(T_{i+1})| < |\mathit{models}(T_i)|$ (case 1 above) or $|O_{i+1}| < |O_i|$ (case 2 above).
This guarantees termination of the algorithm.

To complete the proof of the theorem, we have only to note that for all $J \in \mathit{CIRC}(T,P,Z)$ such that $J \not\models T_{i+1} \setminus T_i$, $J \in \mathit{models}(T_i \cup \{\ell \mid \ell \in I|_{P \cup R}\})$, and therefore $J$ is printed because of Lemmas~\ref{lem:enum}--\ref{lem:disable}.
\hfill
\end{proof}

\section{Implementation and experiments}

Algorithm~\ref{alg:circ} is implemented on top of the SAT solver \textsc{glucose-4.0} \cite{DBLP:conf/ijcai/AudemardS09}, which is extended to natively support cardinality constraints as a special case of the implementation of weight constraints described by \citeANP{DBLP:conf/iclp/GebserKKS09}~\citeyear{DBLP:conf/iclp/GebserKKS09}.
The resulting prototype solver is called \textsc{circumscriptino} (\url{http://alviano.com/software/circumscriptino/}).
Relevant command line parameters are \texttt{-n} and \texttt{--circ-wit}, respectively for limiting the number of models and witnesses to be computed, where a witness of a set $A$ of assumptions is intended as a model $I$ such that $A \subseteq I$.
Intuitively, \texttt{--circ-wit} is used to limit the number of models computed by procedure $\mathit{enumerate}$.
In the special case of \texttt{--circ-wit=1}, line~5 of Algorithm~\ref{alg:circ} is not executed, and model $I$ is directly reported to the user.
This is particularly useful for problems where the interpretation of atoms in $Z$ is not particularly important.
Moreover, the analysis of unsatisfiable cores is preceded by a progression based shrinking \cite{DBLP:journals/tplp/AlvianoD16}.

The performance of the implemented prototype is compared with \textsc{camus\_mcs-1.0.5} \cite{DBLP:journals/jar/LiffitonS08}, \textsc{lbx} \cite{DBLP:conf/ijcai/MenciaPM15}, \textsc{mcsls} (with algorithms \textsc{els} and \textsc{cld}; \citeNP{DBLP:conf/ijcai/Marques-SilvaHJPB13}) and \textsc{hclasp-1.1.5} \cite{DBLP:conf/aaai/GebserKROSW13}.
\textsc{camus\_mcs}, \textsc{lbx} and \textsc{mcsls} are solvers for the enumeration of minimal correction subsets (more details are provided in Section~\ref{sec:related}).
\textsc{hclasp} is a branch of \textsc{clasp} \cite{DBLP:journals/ai/GebserKS12} introducing domain heuristics;
it can enumerate minimal models if atoms of the form \texttt{\_heuristic(p,false,1)} are introduced for each atom $p$ subject to minimization, and if invoked with the command line parameters \texttt{--heuristic=domain --enum-mode=record}.

The experiments comprise two problems, namely the enumeration of minimal intervention strategies and the enumeration of minimal correction subsets.
For the first problem, instances representing biological signaling networks are considered \cite{DBLP:journals/tplp/KaminskiSSV13};
these instances are translated into the input format of \textsc{circumscriptino} thanks to the tool chain \textsc{lp2normal-2.27}+\textsc{lp2atomic-1.17}+\textsc{lp2sat-1.24} \cite{DBLP:conf/birthday/JanhunenN11}.
Regarding the second problem, instances from the SAT Solver Competitions (MUS Special Track) are tested \cite{DBLP:journals/aim/JarvisaloBRS12}.
The experiments were run on an Intel Xeon 2.4 GHz with 16 GB of memory, and time and memory were limited to 10 minutes and 15 GB, respectively.

\begin{table}[b]
    \caption{Enumeration of minimal intervention strategies: execution time in seconds (\textsc{t.o.} for timeout), memory consumption in MB, and number of reported models.}\label{tab:int}
    \centering
    \begin{tabular}{rrrrrrr}
        \toprule
        & \multicolumn{3}{c}{\textsc{circumscriptino}} & \multicolumn{3}{c}{\textsc{hclasp}}\\
        \cmidrule{2-4}\cmidrule{5-7}
        Instance & time & mem & models & time & mem & models\\
        \cmidrule{1-7}
        \textsc{egfr} & 0.00 & 0 & 21 & 0.00 & 0 & 21\\
        \textsc{egfr multiple} & 0.44 & 36 & 83 & 0.10 & 15 & 83\\
        \textsc{tcr} & 8.33 & 16 & 13\,016 & 2.16 & 9 & 13\,016\\
        %TBH6b (intervention strategies of size $\leq 10$) & 4.64 & 1\,248 & 0.33 & 1\,248\\
        \textsc{tbh6b} & \textsc{t.o.} & 95 & 153\,405 & \textsc{t.o.} & 115 & 758\,887\\
        \bottomrule
    \end{tabular}
\end{table}

\begin{figure}
    \figrule
    \begin{tikzpicture}[scale=1]
        %\pgfplotsset{minor grid style={lightgray, opacity=0.2}}
        %\pgfplotsset{major grid style={black, opacity=0.2}}
        %\pgfkeys{/pgf/number format/set thousands separator = {}}
        \begin{axis}[
            scale only axis,
            width=0.7\textwidth,
            height=0.4\textwidth,
            font=\scriptsize,
            legend style={at={(0.26,0.96)}, anchor=north, align=left},
            legend cell align=left,
            xlabel={Number of solved instances},
            ylabel={Running time (seconds)},
            xmin=0,
            xmax=75,
            ymin=0,
            ymax=600,
            xtick={0,15,30,45,60,75},
            ytick={0,120,240,360,480,600},
            grid=both,
        ]
            \addplot [mark size=2pt, color=purple, mark=o] [unbounded coords=jump] table[col sep=tab, skip first n=1, y index=3] {mcs.cactus.csv};
            \addlegendentry{\textsc{hclasp}}

            \addplot [mark size=2pt, color=blue, mark=diamond] [unbounded coords=jump] table[col sep=tab, skip first n=1, y index=2] {mcs.cactus.csv};
            \addlegendentry{\textsc{camus\_mcs}}

            \addplot [mark size=2pt, color=orange, mark=x] [unbounded coords=jump] table[col sep=tab, skip first n=1, y index=4] {mcs.cactus.csv};
            \addlegendentry{\textsc{lbx}}

            \addplot [mark size=2pt, color=brown, mark=pentagon] [unbounded coords=jump] table[col sep=tab, skip first n=1, y index=5] {mcs.cactus.csv};
            \addlegendentry{\textsc{mcsls-els}}

            \addplot [mark size=2pt, color=purple, mark=square] [unbounded coords=jump] table[col sep=tab, skip first n=1, y index=6] {mcs.cactus.csv};
            \addlegendentry{\textsc{mcsls-cld}}

            \addplot [mark size=2pt, color=black, mark=triangle] [unbounded coords=jump] table[col sep=tab, skip first n=1, y index=1] {mcs.cactus.csv};
            \addlegendentry{\textsc{circumscriptino}}

            \addplot [mark size=2pt, color=green!50!black, mark=star] [unbounded coords=jump] table[col sep=tab, skip first n=1, y index=7] {mcs.cactus.csv};
            \addlegendentry{\textsc{virtual best solver}}
        \end{axis}
    \end{tikzpicture}    
    
    \caption{Enumeration of minimal correction subsets: solved instances within a time budget.}\label{fig:mcs:cactus}
    \figrule
\end{figure}
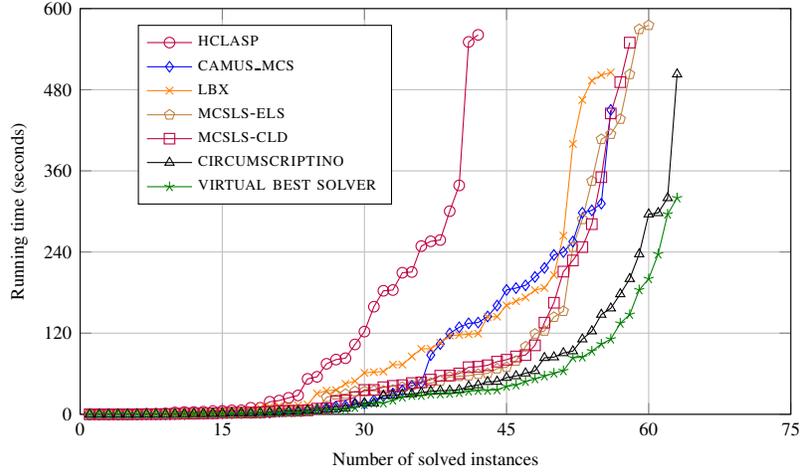

Experimental results on the enumeration of minimal intervention strategies are reported in Table~\ref{tab:int}.
Only 4 instances are available, one of which cannot be solved by the tested solvers;
the other 3 instances, instead, are solved in less than 10 seconds.
For all instances, memory consumption is very low, and \textsc{hclasp} appears to be 4--5 times faster than \textsc{circumscriptino}.
(\textsc{camus\_mcs}, \textsc{lbx} and \textsc{mcsls} are not tested on these instances because their translation is not immediate, and also not the focus of this paper.)

\begin{table}[b]
    \caption{
        Enumeration of minimal correction subsets: solved instances, average execution time in seconds on solved instances, average memory consumption in MB, number of reported models, and average velocity.
        Instances in the first dataset are those solved by at least one solver, while instances in the second dataset are those for which all solvers ran out of time or memory.
    }\label{tab:mcs}
    \begin{tabular}{rrrrrrrrr}
        \toprule
        & \multicolumn{5}{c}{Dataset 1 (63 instances)} & \multicolumn{3}{c}{Dataset 2 (332 instances)} \\
        \cmidrule{2-6}\cmidrule{7-9}
        Solver & sol & time & mem & models & vel & mem & models & vel\\
        \cmidrule{1-9}
        \textsc{circumscriptino} & 63 & 59.3 & 166 & 534\,293 & 1\,875 & 598 & 15\,548\,418 & 77\\
        \textsc{camus\_mcs} & 56 & 78.6 & 411 & 366\,909 & 1\,778 & 1\,442 & 2\,508\,394 & 12\\
        \textsc{hclasp} & 42 & 99.5 & 356 & 351\,785 & 666 & 1\,174 & 13\,988\,100 & 69\\
        \textsc{lbx} & 56 & 98.0 & 146 & 477\,758 & 749 & 636 & 6\,604\,809 & 33\\
        \textsc{mcsls-cld} & 58 & 75.7 & 125 & 489\,240 & 895 & 612 & 7\,538\,687 & 37\\
        \textsc{mcsls-els} & 60 & 91.0 & 125 & 474\,296 & 892 & 594 & 4\,595\,938 & 23\\
        \bottomrule
    \end{tabular}
\end{table}

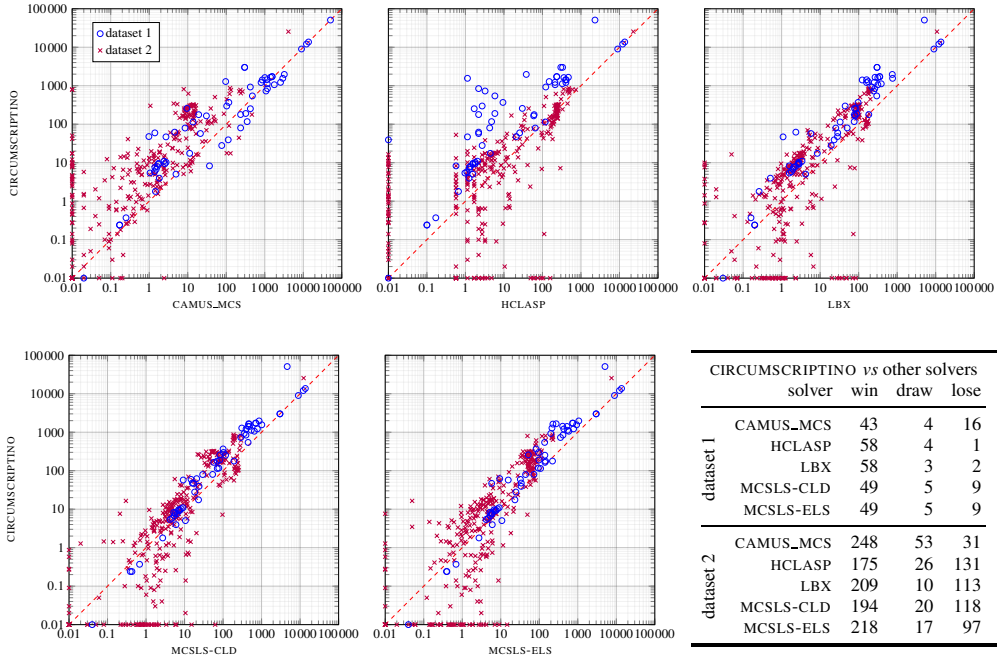
\begin{figure}
    \figrule
    \begin{tikzpicture}[scale=.55]
        \pgfplotsset{minor grid style={lightgray, opacity=0.2}}
        \pgfplotsset{major grid style={black, opacity=0.2}}
%        \pgfplotsset{minor grid style={dashed,red}}
%        \pgfplotsset{major grid style={dotted,green!50!black}}
        \pgfkeys{/pgf/number format/set thousands separator = {}}
        \begin{loglogaxis}[
            %scale only axis,
            width=0.6\textwidth,
            height=0.6\textwidth,
            legend style={at={(0.2,.95)}, anchor=north, align=left},
            legend cell align=left,
            reverse legend,
            xlabel={\textsc{camus\_mcs}},
            ylabel={\textsc{circumscriptino}},
            xmin=.01,
            xmax=100000,
            ymin=.01,
            ymax=100000,
            xtick={.01,.1,1,10,100,1000,10000,100000},
            ytick={.01,.1,1,10,100,1000,10000,100000},
            log ticks with fixed point,
            tick label style={/pgf/number format/1000 sep=\,},
            grid=both,
        ]
            \draw [dashed, red] (rel axis cs:0,0) -- (rel axis cs:1,1);

            \addplot [mark size=2pt, only marks, color=purple, mark=x] table[col sep=tab, skip first n=3, x index=2, y index=1] {mcs-unsol.csv};
            \addlegendentry{dataset 2}
            \addplot [mark size=2pt, only marks, color=blue, mark=o] table[col sep=tab, skip first n=3, x index=2, y index=1] {mcs-sol.csv};
            \addlegendentry{dataset 1}
        \end{loglogaxis}
    \end{tikzpicture}
    \begin{tikzpicture}[scale=.55]
        \pgfplotsset{minor grid style={lightgray, opacity=0.2}}
        \pgfplotsset{major grid style={black, opacity=0.2}}
%        \pgfplotsset{minor grid style={dashed,red}}

%        \pgfplotsset{major grid style={dotted,green!50!black}}
        \pgfkeys{/pgf/number format/set thousands separator = {}}
        \begin{loglogaxis}[
            %scale only axis,
            width=0.6\textwidth,
            height=0.6\textwidth,
            xlabel={\textsc{hclasp}},
            xmin=.01,
            xmax=100000,
            ymin=.01,
            ymax=100000,
            xtick={.01,.1,1,10,100,1000,10000,100000},
            ytick={.01,.1,1,10,100,1000,10000,100000},
            log ticks with fixed point,
            tick label style={/pgf/number format/1000 sep=\,},
            grid=both,
            yticklabels={},
        ]
            \draw [dashed, red] (rel axis cs:0,0) -- (rel axis cs:1,1);

            \addplot [mark size=2pt, only marks, color=purple, mark=x] table[col sep=tab, skip first n=3, x index=3, y index=1] {mcs-unsol.csv};
            \addplot [mark size=2pt, only marks, color=blue, mark=o] table[col sep=tab, skip first n=3, x index=3, y index=1] {mcs-sol.csv};
        \end{loglogaxis}
    \end{tikzpicture}
    \begin{tikzpicture}[scale=.55]
        \pgfplotsset{minor grid style={lightgray, opacity=0.2}}
        \pgfplotsset{major grid style={black, opacity=0.2}}
%        \pgfplotsset{minor grid style={dashed,red}}

%        \pgfplotsset{major grid style={dotted,green!50!black}}
        \pgfkeys{/pgf/number format/set thousands separator = {}}
        \begin{loglogaxis}[
            %scale only axis,
            width=0.6\textwidth,
            height=0.6\textwidth,
            xlabel={\textsc{lbx}},
            xmin=.01,
            xmax=100000,
            ymin=.01,
            ymax=100000,
            xtick={.01,.1,1,10,100,1000,10000,100000},
            ytick={.01,.1,1,10,100,1000,10000,100000},
            log ticks with fixed point,
            tick label style={/pgf/number format/1000 sep=\,},
            grid=both,
            yticklabels={},
        ]
            \draw [dashed, red] (rel axis cs:0,0) -- (rel axis cs:1,1);

            \addplot [mark size=2pt, only marks, color=purple, mark=x] table[col sep=tab, skip first n=3, x index=4, y index=1] {mcs-unsol.csv};
            \addplot [mark size=2pt, only marks, color=blue, mark=o] table[col sep=tab, skip first n=3, x index=4, y index=1] {mcs-sol.csv};
        \end{loglogaxis}
    \end{tikzpicture}

    \bigskip

    \begin{tikzpicture}[scale=.55]
        \pgfplotsset{minor grid style={lightgray, opacity=0.2}}
        \pgfplotsset{major grid style={black, opacity=0.2}}
%        \pgfplotsset{minor grid style={dashed,red}}
%        \pgfplotsset{major grid style={dotted,green!50!black}}
        \pgfkeys{/pgf/number format/set thousands separator = {}}
        \begin{loglogaxis}[
            %scale only axis,
            width=0.6\textwidth,
            height=0.6\textwidth,
            xlabel={\textsc{mcsls-cld}},
            ylabel={\textsc{circumscriptino}},
            xmin=.01,
            xmax=100000,
            ymin=.01,
            ymax=100000,
            xtick={.01,.1,1,10,100,1000,10000,100000},
            ytick={.01,.1,1,10,100,1000,10000,100000},
            log ticks with fixed point,
            tick label style={/pgf/number format/1000 sep=\,},
            grid=both,
        ]
            \draw [dashed, red] (rel axis cs:0,0) -- (rel axis cs:1,1);

            \addplot [mark size=2pt, only marks, color=purple, mark=x] table[col sep=tab, skip first n=3, x index=5, y index=1] {mcs-unsol.csv};
            \addplot [mark size=2pt, only marks, color=blue, mark=o] table[col sep=tab, skip first n=3, x index=5, y index=1] {mcs-sol.csv};
        \end{loglogaxis}
    \end{tikzpicture}
    \begin{tikzpicture}[scale=.55]
        \pgfplotsset{minor grid style={lightgray, opacity=0.2}}
        \pgfplotsset{major grid style={black, opacity=0.2}}
%        \pgfplotsset{minor grid style={dashed,red}}

%        \pgfplotsset{major grid style={dotted,green!50!black}}
        \pgfkeys{/pgf/number format/set thousands separator = {}}
        \begin{loglogaxis}[
            %scale only axis,
            width=0.6\textwidth,
            height=0.6\textwidth,
            xlabel={\textsc{mcsls-els}},
            xmin=.01,
            xmax=100000,
            ymin=.01,
            ymax=100000,
            xtick={.01,.1,1,10,100,1000,10000,100000},
            ytick={.01,.1,1,10,100,1000,10000,100000},
            log ticks with fixed point,
            tick label style={/pgf/number format/1000 sep=\,},
            grid=both,
            yticklabels={},
        ]
            \draw [dashed, red] (rel axis cs:0,0) -- (rel axis cs:1,1);

            \addplot [mark size=2pt, only marks, color=purple, mark=x] table[col sep=tab, skip first n=3, x index=6, y index=1] {mcs-unsol.csv};
            \addplot [mark size=2pt, only marks, color=blue, mark=o] table[col sep=tab, skip first n=3, x index=6, y index=1] {mcs-sol.csv};
        \end{loglogaxis}
    \end{tikzpicture}
    \raisebox{6.5em}{
        \begin{minipage}{.3\textwidth}
            \scriptsize
            \tabcolsep=0.06cm
            \begin{tabular}{rrrrr}
                \toprule
                \multicolumn{5}{c}{\textsc{circumscriptino} \emph{vs} other solvers}\\
                & solver & win & draw & lose\\
                \cmidrule{1-5}
                \parbox[t]{2mm}{\multirow{5}{*}{\rotatebox[origin=c]{90}{dataset 1}}}
                & \textsc{camus\_mcs} & 43 & 4 & 16\\
                & \textsc{hclasp} & 58 & 4 & 1\\
                & \textsc{lbx} & 58 & 3 & 2\\
                & \textsc{mcsls-cld} & 49 & 5 & 9\\
                & \textsc{mcsls-els} & 49 & 5 & 9\\
                \cmidrule{1-5}
                \parbox[t]{2mm}{\multirow{5}{*}{\rotatebox[origin=c]{90}{dataset 2}}}
                & \textsc{camus\_mcs} & 248 & 53 & 31\\
                & \textsc{hclasp} & 175 & 26 & 131\\
                & \textsc{lbx} & 209 & 10 & 113\\
                & \textsc{mcsls-cld} & 194 & 20 & 118\\
                & \textsc{mcsls-els} & 218 & 17 & 97\\
                \bottomrule
            \end{tabular}
        \end{minipage}
    }
    
    \caption{Enumeration of minimal correction subsets: instance-by-instance comparison in terms of velocity (velocity 0 normalized to 0.01). Instances in the first dataset are those solved by at least one solver, while instances in the second dataset are those for which all solvers ran out of time or memory.}\label{fig:mcs:vel}
    \figrule
\end{figure}

Concerning the enumeration of minimal correction subsets, 395 instances are tested.
A cactus plot of the execution time of the tested solvers is reported in Figure~\ref{fig:mcs:cactus}.
The cactus plot also shows the performance of the \emph{virtual best solver}, which is almost aligned with \textsc{circumscriptino}:
only a few seconds and no solved instances are gained when \textsc{circumscriptino} is replaced with the best performant solver in each tested instance.
%Details on the solved instances are given on Table~\ref{tab:mcs:solved}.
The good result of \textsc{circumscriptino} is also confirmed by Table~\ref{tab:mcs}, reporting aggregated results for instances for which at least one solver enumerated all models (dataset~1), and for the remaining instances (dataset~2).
In particular, results for dataset~1 confirm that \textsc{circumscriptino} solves more instances and has the lowest average running time.
Concerning dataset~2, since all solvers ran out of time or memory, a comparison is obtained in terms of \emph{velocity}, defined as number of reported models per second of computation \cite{DBLP:journals/jar/LiffitonS08}.
This metric shows that \textsc{hclasp} produces models with a velocity close to that of \textsc{circumscriptino}, while other solvers are 2-6 times slower.
(The same metric is also applied to dataset~1, where however the average is influenced by instances solved in less than 1 second.)
An instance-by-instance comparison in terms of velocity is shown in the plots in Figure~\ref{fig:mcs:vel}.
Axes are in logarithmic scale, so velocity 0 is normalized to 0.01.
It can be observed that in all cases the majority of points is above the diagonal, meaning that the velocity of \textsc{circumscriptino} is higher than the velocity of other solvers in the majority of instances.
Finally, concerning memory usage, only two memory outs were recorded for \textsc{camus\_mcs} and \textsc{hclasp} on the same instance.

\section{Related work}\label{sec:related}

Circumscription formalizes a preference relation over models of logic theories.
Such a preference relation is essentially subset minimality.
Within this respect, this work is related to many articles in the literature introducing algorithms for computing minimal models.
Among them is the \textsc{optsat} algorithm \cite{DBLP:conf/jelia/GiunchigliaM06}, which is very similar to the algorithm used by \textsc{hclasp} \cite{DBLP:journals/tplp/KaminskiSSV13} in our experiment.
Indeed, \textsc{optsat} essentially modifies the standard heuristic of a SAT solver by selecting the atoms subject to minimization as first branching literals, so to force the search procedure to return a minimal model.
The difference with the approach implemented by \textsc{hclasp} is that \textsc{optsat} fixes an order for the atoms subject to minimization, while the heuristic of \textsc{hclasp} can select any of these atoms;
indeed, the only constraint that the heuristic of \textsc{hclasp} has to satisfy is that all atoms subject to minimization have to be assigned before branching on any other atom.

The main similarity between \textsc{optsat}, \textsc{hclasp} and \textsc{circumscriptino} is that the search starts by trying to falsifying all atoms subject to minimization.
However, as soon as no model falsifying all these atoms exists, the algorithms behave differently:
\textsc{optsat} and \textsc{hclasp} backtrack and flip some of the objective literals, while \textsc{circumscriptino} alters the problem itself so that models falsifying all but one of the original atoms subject to minimization can be searched.

The modification strategy described above is actually the one implemented by many MaxSAT algorithms based on unsatisfiable core analysis \cite{DBLP:journals/constraints/MorgadoHLPM13}.
In Algorithm~\ref{alg:circ}, the unsatisfiable core analysis is performed according to \textsc{one} \cite{DBLP:conf/ijcai/AlvianoDR15}.
This design choice is motivated by the fact that the fresh variables $y_1,\ldots,y_n$ can be later assumed true in order to trivially satisfy the cardinality constraint and the implications introduced by the unsatisfiable core analysis, a feature not required for computing a single solution for a given MaxSAT instance.
Eventually, the algorithm can be adapted to use different unsatisfiable core analysis techniques, in particular \textsc{pmres} \cite{DBLP:conf/aaai/NarodytskaB14} and \textsc{k} \cite{DBLP:conf/ijcai/AlvianoDR15}.

The algorithm implemented by \textsc{camus\_mcs} \cite{DBLP:journals/jar/LiffitonS08} is specifically conceived to address minimal correction subset enumeration, which is also considered in our experimental analysis.
\textsc{camus\_mcs} adds to the input theory a cardinality constraint in order to compute models of bounded size;
such a bound is iteratively increased until all minimal correction sets are computed.
It turns out that \textsc{camus\_mcs} cannot run smoothly on an incremental solver, and in fact some of the learned clauses have to be eliminated when the bound of the cardinality constraint is changed.
Such a drawback also affects the more general algorithm introduced by \citeANP{DBLP:conf/ecai/FaberVCG16}~\citeyear{DBLP:conf/ecai/FaberVCG16}:
an external solver is used to enumerate cardinality minimal solutions of the input problem, and blocking clauses are then added to the theory so that the external solver can be invoked again for enumerating cardinality minimal solutions of the new theory;
the process is repeated until the theory becomes unsatisfiable.

Concerning \textsc{lbx} \cite{DBLP:conf/ijcai/MenciaPM15}, and the algorithms \textsc{els} and \textsc{cld} implemented by \textsc{mcsls} \cite{DBLP:conf/ijcai/Marques-SilvaHJPB13}, all of them follow an iterative approach, where models are improved by performing several calls to a SAT solver.
Specifically, these algorithms start with any assignment, which is used to partition clauses into satisfied $S$ and unsatisfied $U$.
After that, these algorithms iteratively search for a new model satisfying all clauses in $S$ and at least one clause in $U$.
When no further improvement is possible, the last computed model is an MCS of the input theory, which is reported to the user and blocked by means of a blocking clause.
The three algorithms differ in how they enforce an improvement in the current model:
\textsc{els} checks the satisfiability of the theory $S \cup \{c\}$, for some $c \in U$;
\textsc{cld} checks the satisfiability of the theory $S \cup \{d\}$, where $d$ is the disjunction of all literals occurring in $U$;
\textsc{lbx} checks the satisfiability of the theory $S \cup \{\ell\}$, where $\ell$ is some literal occurring in $U$.
The three algorithms additionally take advantage of a few enhancements, such as disjoint core analysis and backbone literals computation.

Another difference between the mentioned algorithms and the one implemented by \textsc{circumscriptino} is represented by the blocking clauses added to the input theory.
In fact, since \textsc{circumscriptino} addresses model enumeration for circumscribed theories with grouping atoms \cite{DBLP:journals/ai/Lifschitz86}, the assignment of non-grouping atoms (those in set $R$) has to be taken into account in the construction of the blocking clause associated with a computed model.

\citeANP{DBLP:journals/ai/LeeL06}~\citeyear{DBLP:journals/ai/LeeL06} studied theoretical properties of the computational problem associated with circumscription.
In particular, they showed that models of circumscribed theories can be computed by adding \emph{loop formulas} to the input theory, where the notion of loop formula is adapted from answer set programming \cite{DBLP:journals/ai/LinZ04,DBLP:conf/iclp/LeeL03}.
Within this respect, the algorithm implemented by \textsc{circumscriptino} requires less additions to the incremental SAT solver.

Finally, subset minimality is among the preferences natively supported in the language of \textsc{asprin} \cite{DBLP:conf/aaai/BrewkaD0S15,DBLP:conf/lpnmr/BrewkaD0S15,DBLP:conf/iclp/0003SW16}, a versatile framework built on top of \textsc{clingo} \cite{DBLP:journals/corr/GebserKKS17}.
The algorithm implemented by \textsc{asprin} is also iterative, meaning that better and better models are computed until an inconsistency arises.
Differently from other iterative algorithms, however, the improvement on the current model is enforced by means of a \emph{preference program}, which is possibly specified by the user in case of custom preferences.
\textsc{asprin} was not tested in the experiment because its performance is clearly bounded by the underlying ASP solver, and therefore by the heuristic algorithm of \textsc{hclasp} in the setting considered in this paper.

\section{Conclusion}

Many practical problems require a preference relation over admissible solutions.
When such a preference amounts to minimize a set of properties, the problem can be naturally represented in circumscription.
Prominent examples of these problems have been considered in our experiments, namely the enumeration of minimal intervention strategies and the enumeration of minimal correction subsets.
The proposed algorithm takes advantage of unsatisfiable core analysis, and showed to be very efficient in many cases.
As a final remark, we stress here that the algorithm presented in this paper can be nicely combined with the tool \textsc{lp2sat} in order to enumerate models of circumscribed answer set programming theories (under the restriction that answer set existence can be checked in NP).

\bibliographystyle{acmtrans}
\bibliography{bibtex}

\label{lastpage}
\end{document}